\documentclass{article}

\usepackage{arxiv}

\usepackage[utf8]{inputenc}
\usepackage{microtype}

\usepackage{amsmath}
\usepackage{amssymb}
\usepackage{amsthm}

\usepackage{booktabs}
\usepackage{array}
\usepackage{tabularx}
\usepackage{graphicx}

\usepackage{listings}
\usepackage{xcolor}
\usepackage{natbib}
\usepackage[colorlinks=true, linkcolor=blue, citecolor=blue, urlcolor=blue, backref=page]{hyperref}
\usepackage{url}

\newtheorem{proposition}{Proposition}

\newcommand{\rltrp}{\texttt{rl-triton}}

\title{\rltrp: High-Performance Triton GPU Kernels\\
       for Reinforcement Learning Credit Assignment}

\author{
  Lars Simon Zehnder\thanks{Correspondence: \texttt{simon@neway.ai}, \texttt{simon.zehnder@gmail.com}. Code: \url{https://github.com/simonsays1980/rl-triton}} \\
  Independent Researcher \\
  Faro, Portugal \\
}

\date{}

\begin{document}

\maketitle

\begin{abstract}
We present \rltrp{}, an open-source library of high-performance GPU kernels for
reinforcement learning credit assignment, implemented in Triton.  The core
contribution is a unified associative scan framework that recasts seven distinct
RL estimation algorithms -- Generalized Advantage Estimation (GAE), V-Trace,
Retrace($\lambda$), TD($\lambda$) returns, discounted returns, eligibility
traces, and episodic prefix sums -- as instances of a single first-order linear
recurrence solved in $O(\log T)$ parallel steps.  All algorithms share the same 
associative scan operator, with algorithm-specific fused Triton kernels constructing 
their recurrence coefficients on-chip.
We verify the associative operator algebraically and define the treatment of 
terminated and truncated episodes explicitly. Benchmarks show a 1.6--5.70$\times$
full-call speedup over a vectorized \texttt{torch.compile} baseline in the
massively parallel simulation regime (thousands of environments, short
rollouts).  The reported range covers all seven algorithms on both GPUs, both with and
without per-step truncation handling.  For most algorithms, 
speedups increase at longer sequence lengths, as the
baseline requires more scan stages as $\log T$ grows, each adding an
intermediate HBM round-trip.  The library is available at 
\url{https://github.com/simonsays1980/rl-triton}.
\end{abstract}

\section{Introduction}

Every iteration of a deep reinforcement learning training pipeline requires
computing \emph{credit assignment} quantities: advantages, value targets, and
return estimates that tell the policy network which actions were good and by how
much.  These quantities include Generalized Advantage Estimation
(GAE) \citep{schulman2016gae}, V-Trace targets \citep{espeholt2018impala},
Retrace($\lambda$) Q-value targets \citep{munos2016retrace}, and several
simpler estimators.  While network forward and backward passes and environment
rollouts typically dominate overall training time, credit assignment
computations are on the critical path of every update step. They operate over 
tensors of shape $[\texttt{num\_envs}, \texttt{seq\_len}]$, where
\texttt{num\_envs} is the number of parallel environments (or workers) and
\texttt{seq\_len}=$T$ is the number of timesteps in each rollout-buffer row,
i.e., the scan length. It is not necessarily the episode length, since a row
may contain multiple episodes separated by boundary flags.  Typical configurations range from
a handful of environments with short rollouts (e.g., 8 environments,
128 steps in standard PPO \citep{schulman2017ppo}) to thousands of environments
with longer horizons in large-scale distributed training
\citep{espeholt2018impala, petrenko2020samplefactory} and RLHF pipelines where
rollout sequences span thousands of tokens \citep{ouyang2022instructgpt}.

All of these estimators share a common mathematical structure: they are
\textbf{first-order linear recurrences}. The five backward estimators have
the form
\begin{equation}
  A_t = \alpha_t + \beta_t \cdot A_{t+1}, \qquad A_T = 0
  \;\text{(or a bootstrap value),}
  \label{eq:recurrence}
\end{equation}
where $\alpha_t$ and $\beta_t$ are algorithm-specific per-step coefficients.
Eligibility traces and episodic prefix sums use the forward-time mirror
$A_t=\alpha_t+\beta_t A_{t-1}$. In either direction, the recurrence creates
a strict temporal dependency chain: a naive sequential implementation
requires $T$ serial steps regardless of how many parallel GPU cores are
available.

Existing RL codebases commonly evaluate this recurrence with a sequential
loop over $t$ (Section~\ref{sec:relatedwork}). Wrapping that loop in
\texttt{torch.compile} removes Python overhead but does not remove its
temporal dependency. An $O(\log T)$ restructuring instead requires expressing
the recurrence as an explicit parallel associative scan \citep{blelloch1990prefix} 
(Section~\ref{sec:scans}). We implement this construction in PyTorch as the
baseline used in our evaluation (Section~\ref{sec:baseline}). Although it has
the same asymptotic depth as the Triton scan, each doubling step still incurs
an HBM round-trip.

\textbf{We implement fused Triton kernels} for these recurrences. The five
backward algorithms use Equation~\ref{eq:recurrence}, while eligibility
traces and episodic prefix sums use its forward-time mirror. Each algorithm
has its own fused kernel, which constructs its coefficients 
(Section~\ref{sec:framework}) from raw rollout inputs on-chip and evaluates 
the recurrence using the same associative scan operator and combine function.

\paragraph{Contributions.}
\begin{enumerate}
  \item A unified algebraic framework showing that seven standard RL credit
        assignment algorithms are instances of a single first-order linear
        recurrence.
  \item Fused Triton kernels sharing one associative scan, with on-chip
        coefficient construction and explicit termination/truncation handling.
  \item Algorithm-specific treatment of bootstrap values at rollout-window
        boundaries -- an additive term in $\alpha_{T-1}$ for some algorithms
        and a nonzero scan carry for others
        (Section~\ref{sec:boundaries}).
  \item A benchmarking harness showing a 1.6--5.70$\times$ full-call speedup
        over a vectorized \texttt{torch.compile} baseline across all seven
        algorithms and both GPUs in the massively parallel simulation regime
        (Section~\ref{sec:results}).
  \item An open-source Python package with correctness tests, performance
        regression tests, and documentation.
\end{enumerate}

\section{Background}

\subsection{The Sequential Dependency Problem}

Consider the GAE recurrence \citep{schulman2016gae}:
\[
  A_t = \delta_t + \gamma\lambda \cdot A_{t+1},
\]
where $\delta_t = r_t + \gamma V(s_{t+1}) - V(s_t)$ is the one-step TD error.
Computing $A_0$ requires $A_1$, which requires $A_2$, all the way to
$A_{T-1}$.  This is a strict sequential chain of length $T$.

On GPU hardware, sequential chains of length $T$ are expensive not because of
arithmetic cost but because of \textbf{memory traffic}: a naive loop must read
$\delta_t$ and $\gamma\lambda$ from HBM, perform the addition and
multiplication, and write $A_t$ back to HBM -- one full round-trip per
timestep.  For $T{=}1024$, $\texttt{num\_envs}{=}128$, this means 1024 serial
HBM round-trips even though 128 environment rows could in principle be
processed in parallel.

\texttt{torch.compile} cannot remove this dependency: it eliminates
per-step Python interpreter overhead, but a compiler cannot break a true
data dependency the algorithm itself imposes, so the loop still executes
$T$ sequential steps.  This uncompiled sequential loop -- the pattern used
by CleanRL, RLlib, and similar codebases (Section~\ref{sec:relatedwork}) --
is reported as the \emph{Loop} baseline in our results (Section~\ref{sec:results}).

Removing the $O(T)$ chain requires restructuring the \emph{algorithm}, not
just compiling it: rewriting the recurrence as an associative scan
(Section~\ref{sec:scans} below) so that \texttt{torch.compile} has a
$\log_2 T$-depth computation to compile in the first place, rather than a
$T$-depth one.  This is the vectorized, \texttt{torch.compile}'d doubling-scan
baseline used throughout this paper's evaluation (Section~\ref{sec:baseline})
and is a
substantially stronger baseline than the naive loop -- it achieves the same
$O(\log T)$ depth as our Triton kernel's internal scan (Section~\ref{sec:fused})
-- but, as Section~\ref{sec:baseline} details, still pays a full HBM
round-trip per doubling step rather than keeping intermediate results
on-chip.

\subsection{Parallel Associative Scans}
\label{sec:scans}

The parallel associative scan is a foundational algorithm in parallel computing
\citep{blelloch1990prefix,ladner1980parallel}.  Given an array of elements
$[x_0, x_1, \ldots, x_{T-1}]$ and an associative binary operator $\oplus$,
the scan computes all prefix reductions $[x_0,\, x_0 \oplus x_1,\, \ldots,\,
x_0 \oplus \cdots \oplus x_{T-1}]$ in $O(\log T)$ parallel steps on $T$
processors, using a binary tree reduction.

The first-order linear recurrence $A_t = \alpha_t + \beta_t \cdot A_{t+1}$ 
can be expressed as a prefix scan of tuples $(\alpha_t, \beta_t)$ under the 
associative operator:
\begin{equation}
  (\alpha_B, \beta_B) \oplus (\alpha_A, \beta_A)
    = \bigl(\alpha_B + \beta_B\,\alpha_A,\;\; \beta_A\,\beta_B\bigr).
  \label{eq:combine}
\end{equation}

\begin{proposition}[Associativity of $\oplus$]
The operator defined by Equation~\ref{eq:combine} is associative.
\end{proposition}
\begin{proof}
For any three tuples $(\alpha_C, \beta_C)$, $(\alpha_B, \beta_B)$,
$(\alpha_A, \beta_A)$:
\begin{align*}
  \bigl[(\alpha_C,\beta_C)\oplus(\alpha_B,\beta_B)\bigr]\oplus(\alpha_A,\beta_A)
  &= (\alpha_C+\beta_C\alpha_B,\;\beta_B\beta_C)\oplus(\alpha_A,\beta_A) \\
  &= (\alpha_C+\beta_C\alpha_B+\beta_B\beta_C\alpha_A,\;\beta_A\beta_B\beta_C).
\end{align*}
\begin{align*}
  (\alpha_C,\beta_C)\oplus\bigl[(\alpha_B,\beta_B)\oplus(\alpha_A,\beta_A)\bigr]
  &= (\alpha_C,\beta_C)\oplus(\alpha_B+\beta_B\alpha_A,\;\beta_A\beta_B) \\
  &= (\alpha_C+\beta_C(\alpha_B+\beta_B\alpha_A),\;\beta_A\beta_B\beta_C) \\
  &= (\alpha_C+\beta_C\alpha_B+\beta_C\beta_B\alpha_A,\;\beta_A\beta_B\beta_C).
\end{align*}
Both expressions are equal.
\end{proof}

For the five backward recurrences, the kernel loads the sequence in reverse
chronological order: reversed index $k=0$ corresponds to timestep $t=T-1$,
while $k=T-1$ corresponds to $t=0$. The associative scan therefore operates
in late-to-early time order, and the result at reversed index $k$ corresponds
to chronological timestep $t=T-1-k$. Eligibility traces and episodic prefix
sums instead scan the sequence in its original chronological order.

\subsection{Triton and the GPU Memory Hierarchy}

Triton \citep{tillet2019triton} is a domain-specific language and compiler for
writing GPU kernels in Python-like syntax.  It abstracts over warp-level
programming while exposing the key performance-relevant features of the GPU
memory hierarchy:

\begin{itemize}
  \item \textbf{HBM (High Bandwidth Memory):} Main GPU memory,
        $\sim$2--3.35\,TB/s bandwidth on datacenter HBM3 parts such as the
        80GB H100 used in this paper's evaluation, latency on the order of
        several hundred cycles.  Accessible by all SMs  (The RTX~2000~Ada also evaluated
        here uses 16GB of GDDR6 instead of HBM, at substantially lower
        bandwidth -- Section~\ref{sec:results} reports results on both GPUs).
  \item \textbf{SRAM (Shared Memory):} Per-SM scratchpad, $\sim$5--10$\times$
        higher bandwidth than HBM, very low latency.  Configurable capacity
        up to $\sim$100\,KB per SM on Ada-generation SMs (RTX~2000~Ada) and
        up to $\sim$228\,KB per SM on H100.
  \item \textbf{Registers:} Per-thread storage, zero-latency access.
        Arithmetic occurs here.  64K 32-bit registers per SM on both GPUs
        evaluated in this paper, shared across all resident threads.
\end{itemize}

Credit assignment kernels are memory-bandwidth bound rather than compute
bound \citep{williams2009roofline}: the arithmetic intensity (FLOPs per byte
of HBM traffic) is low, so reducing HBM accesses directly reduces runtime.
The performance advantage of a fused kernel therefore comes from
\textbf{minimizing HBM round-trips} \citep{dao2022flashattention}.  A fused 
implementation keeps the scan state on-chip between stages, avoiding
the intermediate HBM round-trips required by a materialized scan. The associative 
scan keeps intermediate tuples on-chip throughout the $O(\log T)$ reduction, 
with no HBM round-trip until the final output store.

Triton's \texttt{tl.associative\_scan} primitive maps directly to this pattern,
allowing the kernel to express the tree reduction at a high level while the
Triton compiler generates efficient hardware-specific reduction instructions
for both NVIDIA (warp shuffles) and AMD (wavefront operations) GPUs.

\section{The Unified Linear Recurrence Framework}
\label{sec:framework}

\subsection{The Shared Scan Operator}

All seven algorithms use the same associative scan operator and combine
function. For the five backward recurrences, the kernel presents
$(\alpha_t,\beta_t)$ to the scan in reverse chronological order and writes
the result back in original time order. Eligibility traces and episodic
prefix sums instead scan the original sequence in forward order.

For the backward algorithms, the boundary carry $A_T$ is
algorithm-dependent rather than a universal constant;
Section~\ref{sec:boundaries} derives this condition, and
Table~\ref{tab:algorithms} lists it for each algorithm. The scan incorporates
this carry after evaluating the recurrence with
\texttt{tl.associative\_scan} and the combine function in
Equation~\ref{eq:combine}. The forward scans use the same combine function
but require no successor-value bootstrap.

For the common case ($\texttt{seq\_len} \le 131072$), each Python API
dispatches to an algorithm-specific fused kernel that receives the raw
rollout tensors (rewards, values, done flags, importance ratios, etc.) and
constructs $\alpha_t$ and $\beta_t$ in registers; no PyTorch operation
materializes them as intermediate tensors. For longer sequences, the
backward-scan algorithms use the unfused chunked fallback described in
Section~\ref{sec:chunked}. Table~\ref{tab:algorithms} summarizes the
per-algorithm mapping from rollout quantities to $\alpha_t$ and $\beta_t$.

\begin{table}[t]
\centering
\caption{Algorithm-specific coefficients for the shared affine scan,
         applied in reverse or forward time as appropriate.
         $d^{\mathrm{term}}_t, d^{\mathrm{trunc}}_t \in \{0,1\}$ flag
         termination and truncation; $d_t = d^{\mathrm{term}}_t \vee
         d^{\mathrm{trunc}}_t$ is the flag used in $\beta_t$, combining
         both.  $\Delta_t = v_t - V(s_t)$ is the V-Trace value
         delta (Section~\ref{sec:vtrace}).  The \textbf{Bootstrap} column
         gives each algorithm's scan boundary carry (Section~\ref{sec:boundaries}
         explains why it is $0$ for GAE/V-Trace/Retrace($\lambda$) but
         nonzero for TD($\lambda$)/discounted returns).  Eligibility traces
         use vector notation since the trace is vector-valued; see
         Section~\ref{sec:results-limitations} for per-call scalar
         handling.}
\label{tab:algorithms}
\renewcommand{\arraystretch}{1.4}
\small
\begin{tabularx}{\textwidth}{@{}l X X l@{}}
\toprule
\textbf{Algorithm} & $\boldsymbol{\alpha_t}$ & $\boldsymbol{\beta_t}$ & \textbf{Bootstrap} \\
\midrule
GAE
  & $r_t + \gamma(1-d^{\mathrm{term}}_t)V(s_{t+1}) - V(s_t)$
  & $\gamma\lambda(1-d_t)$
  & $A_T = 0$ \\
V-Trace
  & $\rho_t\!\left(r_t + \gamma(1-d^{\mathrm{term}}_t)V(s_{t+1}) - V(s_t)\right)$
  & $\gamma c_t(1-d_t)$
  & $\Delta_T = 0$ \\
Retrace($\lambda$)
  & $r_t + \gamma(1-d^{\mathrm{term}}_t)\mathbb{E}_\pi[Q(s_{t+1},\cdot)] - Q(s_t,a_t)$
  & $\gamma c_{t+1}(1-d_t)$
  & $0$ \\
TD($\lambda$)
  & $r_t + \gamma(1-\lambda)(1-d^{\mathrm{term}}_t)V(s_{t+1})$
  & $\gamma\lambda(1-d_t)$
  & $G^\lambda_T$ \\
Disc.\ Returns
  & $r_t$
  & $\gamma(1-d_t)$
  & $G_T$ \\
Elig.\ Traces (fwd)
  & $\nabla_{\mathbf{w}}\hat{V}(s_t,\mathbf{w}_t)$
  & $\gamma\lambda(1-d_{t-1})$
  & $\mathbf{z}_{-1} = \mathbf{0}$ \\
Prefix Sum (fwd)
  & $x_t$
  & $1 - d_{t-1}$
  & $0$ \\
\bottomrule
\end{tabularx}
\end{table}

\subsection{Episode and Rollout-Window Boundary Semantics}
\label{sec:boundaries}

Credit-assignment quantities must not propagate across episode boundaries
within a rollout buffer. \rltrp{} distinguishes termination and truncation,
which both sever the scan carry but differ in their bootstrap semantics, from
the rollout-window boundary at $t=T-1$. Eligibility traces and episodic prefix
sums do not use value bootstrapping; for these algorithms, episode boundaries
only reset the recurrence.

\paragraph{Terminated steps} ($d^{\mathrm{term}}_t = 1$).
The episode ended naturally. For algorithms that bootstrap from a successor
value, $V(s_{t+1})$ is conceptually zero and is therefore removed from the
local update through $(1-d^{\mathrm{term}}_t)$. The decay $\beta_t$ is also
zeroed through $(1-d_t)$, preventing the scan carry from crossing into the
next episode.

\paragraph{Truncated steps} ($d^{\mathrm{trunc}}_t = 1$).
The episode was cut short but continues in principle. The value stored at
$t+1$ belongs to the next episode in the rollout buffer and therefore cannot
serve as the continuation value. For algorithms that bootstrap from
$V(s_{t+1})$, the caller instead supplies the correct continuation through
\texttt{bootstrap\_values[env, t]}; this contribution is incorporated into
$\alpha_t$ with the appropriate discount. As for termination,
$\beta_t=0$ severs the scan carry, but does not remove the bootstrap already
encoded in $\alpha_t$.

\paragraph{Rollout-window boundary} (final step $t=T-1$).
If the rollout ends in the middle of an episode, its successor state $s_T$
lies outside the buffer. For GAE, V-Trace, and Retrace($\lambda$), the supplied
$V(s_T)$ enters the local update at $t=T-1$, while the scan boundary carry
remains
\[
    A_T = 0.
\]
A nonzero carry would represent continuation of the trace recurrence beyond
the sampled window; the required one-step bootstrap has already been included
in $\alpha_{T-1}$. For TD($\lambda$) and discounted returns, the supplied continuation value at
$s_T$ is instead distributed between $\alpha_{T-1}$ and the boundary carry
$A_T$, with coefficients that sum to the required discounted continuation
(Section~\ref{sec:tdlambda}). Eligibility traces and episodic prefix sums do
not require a successor-value bootstrap.

If $t=T-1$ is itself truncated, it is handled as a truncation rather than as
an ordinary non-terminal window edge: the supplied continuation value enters
through the truncation term in $\alpha_{T-1}$, while
$\beta_{T-1}=0$ suppresses any contribution from the scan carry. The bootstrap
is therefore not double-counted.

Thus, termination and truncation both reset the recurrence, but termination
uses a zero successor value whereas truncation uses a supplied continuation
value. A non-terminal rollout-window edge does not represent an episode
boundary; its continuation value enters through $\alpha_{T-1}$, the boundary
carry $A_T$, or both, depending on the algorithm
(Section~\ref{sec:algorithms}).

\subsection{Algorithm-Specific Details}
\label{sec:algorithms}

The following subsections specify $\alpha_t$ and $\beta_t$ for each algorithm.
Episode-boundary semantics -- how terminated steps, truncated steps, and the
window boundary are handled -- are common to all algorithms and described in
Section~\ref{sec:boundaries}; they are not repeated here.

\subsubsection{Generalized Advantage Estimation (GAE)}

GAE \citep{schulman2016gae} computes the advantage as the exponentially-weighted
mixture of $n$-step TD residuals:
\[
  A_t^{\mathrm{GAE}} = \sum_{k=0}^{\infty}(\gamma\lambda)^k\delta_{t+k}.
\]
This telescopes to the backward recurrence $A_t = \delta_t + \gamma\lambda \cdot
A_{t+1}$.  With episode-boundary masking:
\[
  \alpha_t=\delta_t = r_t + \gamma(1-d^{\mathrm{term}}_t)V(s_{t+1}) - V(s_t),
  \qquad
  \beta_t = \gamma\lambda(1-d_t).
\]
The $\lambda=0$ case reduces GAE to a pure one-step TD advantage (low variance,
high bias); $\lambda=1$ yields the full Monte Carlo return minus the baseline
(low bias, high variance).

\subsubsection{V-Trace}
\label{sec:vtrace}

V-Trace \citep{espeholt2018impala} corrects for policy lag in distributed
training by weighting each TD error with a clipped importance sampling (IS)
ratio.  Let $\pi$ be the current learner policy and $\mu$ the behaviour policy
that collected the data.  Define the per-step IS ratio
$\rho^{\mathrm{raw}}_t = \pi(a_t|s_t)/\mu(a_t|s_t)$, then clip it at two
separate thresholds:
\[
  \rho_t = \min\!\left(\bar{\rho},\,\rho^{\mathrm{raw}}_t\right),
  \qquad
  c_t   = \min\!\left(\bar{c},\,\rho^{\mathrm{raw}}_t\right).
\]
$\bar{\rho}$ controls the bias--variance trade-off of the value target:
a smaller $\bar{\rho}$ limits how far the target can deviate from the
on-policy value, reducing variance at the cost of some bias.
$\bar{c}$ controls the multi-step trace length: it multiplies along the
return trajectory in $\beta_t = \gamma c_t(1-d_t)$, so a smaller $\bar{c}$
shortens the effective lookahead, trading variance for bias.
Typical defaults are $\bar{\rho} = 1$ and $\bar{c} = 1$
\citep{espeholt2018impala}.
The V-Trace target delta satisfies $\Delta_t = \delta^V_t + \gamma c_t \cdot
\Delta_{t+1}$, mapping to:
\[
  \alpha_t = \delta^V_t = \rho_t\!\left(r_t + \gamma(1-d^{\mathrm{term}}_t)V(s_{t+1}) - V(s_t)\right),
  \qquad
  \beta_t = \gamma c_t(1-d_t).
\]
After the scan, targets are formed as $v_t = \Delta_t + V(s_t)$ and advantages
as $A_t = \rho_t(r_t + \gamma v_{t+1} - V(s_t))$, all computed within the same
fused kernel.

\subsubsection{Retrace(\texorpdfstring{$\lambda$}{lambda})}

Retrace($\lambda$) \citep{munos2016retrace} extends V-Trace to Q-value
estimation with an \emph{index shift}: the decay coefficient applying to
$\Delta_{t+1}$ is $c_{t+1}$ (next-step importance weight) rather than $c_t$:
\[
  \Delta_t = \delta_t + \gamma c_{t+1}\cdot\Delta_{t+1}.
\]
This shift reflects that Retrace evaluates a state-action pair $(s_t,a_t)$:
the action $a_t$ is fixed as the subject of evaluation, so off-policy
correction begins only from the next action $a_{t+1}$ onward. Because the
kernel already loads $(\alpha_t,\beta_t)$ in reverse chronological order
(Section~\ref{sec:framework}), obtaining $c_{t+1}$ requires no separate
pre-shifted copy of the action-probability arrays. The same reverse-order load,
offset by one position, provides $\pi(a_{t+1}\mid s_{t+1})$ and
$\mu(a_{t+1}\mid s_{t+1})$ directly in registers alongside
$\pi(a_t\mid s_t)$ and $\mu(a_t\mid s_t)$
(Section~\ref{sec:fused}).

The boundary carry is fixed at zero, $\Delta_T=0$.
The expected Q-values at $s_T$ are already encoded in $\alpha_{T-1}$.
Moreover, $\beta_{T-1}=\gamma c_T(1-d_{T-1})$ contains $c_T$, which would
require the behaviour-policy probability of an action at $s_T$ that was never
sampled; we therefore take $c_T=0$. Interior values $\Delta_t$ are nonzero in
general and are computed by the scan.

\subsubsection{TD(\texorpdfstring{$\lambda$}{lambda}) Returns and Discounted Returns}
\label{sec:tdlambda}

TD($\lambda$) returns \citep{sutton2018rl} interpolate between one-step TD
($\lambda=0$) and Monte Carlo ($\lambda=1$):
\[
  G^\lambda_t = r_t + \gamma\!\left[(1-\lambda)V(s_{t+1}) + \lambda G^\lambda_{t+1}\right].
\]
Expanding and grouping into $\alpha_t + \beta_t \cdot G^\lambda_{t+1}$ form:
\[
  \alpha_t = r_t + \gamma(1-\lambda)(1-d^{\mathrm{term}}_t)V(s_{t+1}),
  \qquad
  \beta_t = \gamma\lambda(1-d_t).
\]
Discounted returns \citep{sutton2018rl} are the special case $\lambda=1$ with
no value function: $\alpha_t = r_t$, $\beta_t = \gamma(1-d_t)$.  Two boundary
cases require separate treatment.  At \emph{interior truncated steps}
($d^{\mathrm{trunc}}_t=1$, $t<T{-}1$), $\beta_t=0$ severs the carry, so the
supplied continuation value must enter $\alpha_t$ directly. For
TD($\lambda$), this gives
$\alpha_t=r_t+\gamma V(s^{\mathrm{true}}_{t+1})$.  At the \emph{window
boundary} ($t = T{-}1$, continuing episode), $\beta_{T-1} = \gamma\lambda$ and
the bootstrap is passed as the scan carry $G_T = V(s_T)$, entering through
both terms at once:
\[
  G_{T-1} = \underbrace{r_{T-1} + \gamma(1-\lambda)V(s_T)}_{\alpha_{T-1}}
          + \gamma\lambda\,V(s_T)
          = r_{T-1} + \gamma V(s_T).
\]
The two coefficients on $V(s_T)$ total $\gamma$, pricing the bootstrap in
exactly once.

\subsubsection{Eligibility Traces}

Eligibility traces are one of the two \textbf{forward} recurrences in the
library:
\[
  \mathbf{z}_t = \gamma\lambda(1-d_{t-1})\,\mathbf{z}_{t-1}
                 + \nabla_{\mathbf{w}}\hat{V}(s_t,\mathbf{w}_t), \qquad d_{-1} := 0.
\]
This maps to $\alpha_t = \nabla_{\mathbf{w}}\hat{V}(s_t,\mathbf{w}_t)$,
$\beta_t = \gamma\lambda(1-d_{t-1})$: the carry into $t$ is severed by the
\emph{preceding} step's flag, since $d_t$ marks where an episode ends, not
where the next one starts.  Unlike the backward algorithms, the sequence is
loaded in \emph{chronological} order ($t = 0 \to T{-}1$, past to present) so
each position accumulates a weighted sum of past gradients rather than
future TD errors.  The same combine function applies; only the load order
differs.  For linear function approximation,
$\nabla_{\mathbf{w}}\hat{V}(s_t,\mathbf{w}_t) = x_t$ (the feature vector).

\subsubsection{Episodic Prefix Sum}

The episodic prefix sum (segmented scan) is the second \textbf{forward} scan and computes a running total that resets
at segment boundaries.  Two boundary conventions are supported via a \texttt{boundary} argument,
depending on whether the flag marks the end of the preceding segment or the
start of the new one.  The default, \texttt{boundary="ends\_at"}, matches the rollout convention used
by the other kernels in this library:
\[
  C_t = x_t + (1-d_{t-1})\,C_{t-1}, \qquad d_{-1} := 0,
\]
mapping to $\alpha_t = x_t$, $\beta_t = 1-d_{t-1}$.  Passing
\texttt{boundary="starts\_at"} instead resets \emph{at} $t$
($\beta_t = 1-d_t$), for callers whose boundary flag marks a new segment's
first element directly rather than an ending one.

\section{Implementation}

\subsection{Kernel Architecture}

Each environment row is processed by one Triton program instance, indexed by
\texttt{tl.program\_id(0)}. Within the program, a block of
\texttt{BLOCK\_SIZE} timestep elements is distributed across the program's
warps and processed in registers and on-chip memory \citep{kirk2016programming}. Thus, parallelism is
exposed both across environment rows through the launch grid and across
timesteps within each program.

The generic backward kernel consumes precomputed $\alpha$ and $\beta$ tensors
and is used by the chunked fallback in Section~\ref{sec:chunked}. For
$\texttt{seq\_len}\le131072$, the algorithm-specific fused kernels instead
load the raw rollout tensors and construct $\alpha_t$ and $\beta_t$ in-register,
avoiding their materialization in HBM. Both paths then follow the same scan
procedure:
\begin{enumerate}
\item \textbf{Load} (HBM $\to$ registers): Load the required inputs for
      $\texttt{BLOCK\_SIZE}$ timesteps, in reverse order for backward scans,
      using \texttt{tl.arange}-based offsets. Masked loads handle positions
      beyond the actual sequence length.

\item \textbf{Construct $\mathbf{\alpha}$, $\mathbf{\beta}$} (registers): Fused kernels derive
      $\alpha_t$ and $\beta_t$ from the raw inputs; the generic kernel receives
      them directly.

\item \textbf{Scan} (on-chip): Apply
      \texttt{tl.associative\_scan(($\alpha$, $\beta$), axis=0,
      combine\_fn=\_combine)}, giving $O(\log \texttt{BLOCK\_SIZE})$
      dependency depth. Out-of-range positions are excluded by the load/store
      masks.

\item \textbf{Bootstrap} (registers): Incorporate the algorithm-dependent
      boundary carry $A_T$ as
      \texttt{scan\_result + decay\_product * bootstrap}
      (Section~\ref{sec:boundaries}).

\item \textbf{Store} (registers $\to$ HBM): Write the results in original
      time order using the reversed offsets.
\end{enumerate}

\subsection{The Fused Kernel Pattern}
\label{sec:fused}

Fusion eliminates the HBM materialization of $\alpha$ and $\beta$ between
preprocessing and the scan. Table~\ref{tab:hbm} illustrates the resulting
traffic reduction for GAE with bootstrap values; the same pattern applies to
the other fused kernels.

\begin{table}[h]
\centering
\caption{HBM accesses for GAE (with bootstrap values, no truncations).
  $r$: rewards; $v$: values; $d$: terminateds; $bs$: bootstrap values;
  $\delta$: TD residuals; $\beta$: decay scalars; $A$: advantages.}
\label{tab:hbm}
\begin{tabularx}{\linewidth}{Xcc}
\toprule
Step & Reads & Writes \\
\midrule
\multicolumn{3}{l}{\textit{Fused pipeline (single Triton kernel)}} \\
\quad Load $r$, $v$ (read twice at offsets \texttt{rev} and \texttt{rev+1}),
  $d$, $bs$; compute $\alpha$, $\beta$ in-register; write $A$ & 5 & 1 \\
\quad \textbf{Total} & \multicolumn{2}{c}{\textbf{6}} \\
\midrule
\multicolumn{3}{l}{\textit{Unfused pipeline (three GPU operations)}} \\
\quad Compute $\delta = r + \gamma v_{t+1}(1{-}d) - v_t$ ($\delta$ is GAE's
  $\alpha_t$; $v_{t+1}$ reads $bs$ at the boundary column, a 5th tensor) & 5 & 1 \\
\quad Compute $\beta = \gamma\lambda(1{-}d)$ & 1 & 1 \\
\quad Scan: $(\delta,\beta) \to A$ & 2 & 1 \\
\quad \textbf{Total} & \multicolumn{2}{c}{\textbf{11}} \\
\bottomrule
\end{tabularx}
\end{table}

\noindent \noindent
For GAE, fusion reduces the counted HBM accesses from 11 to 6 by eliminating
the write and subsequent read of $\delta$ and $\beta$. Compile-time
\texttt{constexpr} flags further specialize common cases:
\begin{itemize}
  \item \texttt{HAS\_TRUNCATIONS=False}: skips the truncateds array and
        the 2D \texttt{bootstrap\_values[num\_envs, seq\_len]} tensor (used
        for per-step truncation bootstraps), saving two HBM reads.  The
        window-boundary bootstrap is still supported via a separate 1D
        scalar argument \texttt{bootstrap\_ptr[num\_envs]}, controlled by
        \texttt{HAS\_BOOTSTRAP}.
  \item \texttt{HAS\_BOOTSTRAP=False}: substitutes \texttt{0.0} for the
      window-boundary bootstrap and avoids materializing an all-zero
      \texttt{bootstrap\_ptr}. Relative to the previous path, this reduces
      full-call time by ${\sim}20$--$27\%$ at small shapes on both GPUs
      evaluated here.\footnote{The isolated benchmark compares the previous
      \texttt{torch.zeros(num\_envs)} path with the current
      \texttt{HAS\_BOOTSTRAP=False} path, saving ${\sim}8$--$10\,\mu$s
      (\texttt{benchmarks/measure\_bootstrap\_skip.py}).}
\end{itemize}

\subsection{Chunked Fallback for Long Sequences}
\label{sec:chunked}

The flat kernel is limited to $\texttt{seq\_len} \le 131072$ (the maximum
\texttt{BLOCK\_SIZE} supported by Triton's \texttt{tl.associative\_scan}).  For
longer sequences, backward algorithms (GAE, V-Trace, Retrace, TD($\lambda$),
discounted returns) fall back to a chunked PyTorch path that processes the
sequence in blocks and propagates the carry between blocks.  The eligibility trace 
and prefix-sum forward scans do not have a chunked fallback: a
left-to-right chunked forward scan is structurally feasible but was not
implemented because sequences longer than $131072$ timesteps are very
uncommon in on-policy RL, making the added implementation complexity hard
to justify.

The chunked path is unfused: it materializes $\alpha$ and $\beta$ as full
tensors before invoking the scan.  At $\texttt{seq\_len} > 131072$,
throughput is dominated by global memory bandwidth rather than kernel launch
overhead, so the fusion benefit is smaller and the additional implementation
complexity of a fused chunked kernel is not justified for typical RL workloads.

\subsection{Numerical Precision}
\label{sec:precision}

All kernels require \texttt{float32} inputs and raise a \texttt{TypeError} on
other dtypes.  The scan accumulates $T$ additions over potentially thousands of
timesteps; \texttt{bfloat16}'s 7 mantissa bits cause meaningful numerical drift
at large $T$ compared to \texttt{float32}'s 23 mantissa bits.  In
mixed-precision training pipelines (\texttt{torch.autocast}), the recommended
pattern is to cast advantage inputs to \texttt{float32} before the scan and
allow the policy forward pass to run in \texttt{bfloat16}.

\section{Correctness Analysis}

\subsection{Algebraic Verification}
\label{sec:algebraic-verification}

Because $\oplus$ is associative (Section~\ref{sec:scans}), the $T$
per-timestep tuples can be reduced in any grouping without changing the
result: a $\log_2(T)$-depth tree reduction yields the same accumulation as
the sequential scan, in $\log_2(T)$ dependent steps rather than $T$.
Concretely, for a 4-step window the reduction ends with one thread holding
a single accumulated tuple spanning all four timesteps; the $\alpha$
component of that tuple works out to
\[
  A_0 = \alpha_0 + \beta_0\bigl(\alpha_1 + \beta_1(\alpha_2 + \beta_2\alpha_3)\bigr),
\]
the exact sequential expansion of $A_t = \alpha_t + \beta_t A_{t+1}$ -- the value 
at the first chronological timestep, depending on all four local recurrence terms.  Every
combine pairs a chronologically-earlier tuple with a chronologically-later
one, per Equation~\ref{eq:combine}, so
Section~\ref{sec:boundary-correctness}'s boundary-severing argument (stated
directly in these same $t$/$t{+}1,\ldots$ terms) applies here unchanged.
Appendix~\ref{sec:appendix-scan-trace} gives the full thread-by-thread trace.

\subsection{Episode Boundary Correctness}
\label{sec:boundary-correctness}

When $d_t = 1$ at a boundary step, $\beta_t = 0$, and the operator severs
the carry.  Combining step $t$'s own tuple $(\alpha_t, 0)$ with
$(\alpha_{t+1..}, \beta_{t+1..})$ -- the tuple already accumulated for
steps $t{+}1,\ldots$ -- gives
\[
  (\alpha_t, 0) \oplus (\alpha_{t+1..}, \beta_{t+1..})
  = (\alpha_t + 0\cdot\alpha_{t+1..},\;0\cdot\beta_{t+1..})
  = (\alpha_t,\;0).
\]
Everything accumulated for later timesteps is multiplied by zero, and the
resulting $\beta$ is zero, so no credit propagates backward past $t$.  The
scan result at $t$ is therefore unaffected by any following episode across
a done flag.

The tree reduction therefore reproduces the sequential recurrence, while 
$\beta_t=0$ prevents propagation across episode boundaries. For the two 
forward scans, the same argument applies with the shifted boundary coefficient: 
when $d_{t-1}=1$, $\beta_t=0$ prevents carry from the preceding episode from 
entering timestep $t$.

\section{Benchmarking Methodology}

\subsection{Measurement Protocol}
\label{sec:measurement}

All benchmarks use CUDA event timing with the following protocol:
\begin{enumerate}
  \item \textbf{Warmup} (20 iterations, untimed): absorbs Triton JIT
        compilation, autotuning, and first-touch allocation.
  \item \textbf{Timed iterations} (50 per trial, 5 trials): each iteration is
        individually timed with CUDA event synchronization before start and
        after stop.
  \item \textbf{Robust estimator}: each trial reports its median over 50
        iterations; the final result is the minimum of 5 trial medians.
\end{enumerate}

Credit-assignment errors are often finite and plausible-looking rather than
crashes -- for example, seeding the window carry with the bootstrap for GAE
would double-count $V(s_T)$ at every position (Section~\ref{sec:boundaries})
-- so they survive testing that checks only for NaN/Inf.  Every kernel is
therefore validated against an independent sequential reference
implementation (\texttt{atol=rtol=1e-4}) before timing results are included 
in the evaluation.

Following established parallel benchmarking practice
\citep{hoefler2015benchmarking}, we report the minimum of five independently
warmed trial medians to reduce sensitivity to occasional interference from
frequency changes or OS scheduling. Such fixed-overhead events
disproportionately affect short kernels and can otherwise produce
non-monotonic measurements across problem sizes.

All GPU benchmarks (both the H100 and the RTX~2000~Ada in
Section~\ref{sec:results}) were run on RunPod-hosted cloud instances, on
\texttt{torch==2.4.1+cu124} with Triton 3.0.0.

\subsection{Baseline}
\label{sec:baseline}

The baseline uses a linear-space scan:
\texttt{parallel\_suffix\_scan}/\texttt{parallel\_prefix\_scan}, a
$\log_2(T)$-doubling associative scan with no $\log$/$\exp$
anywhere,\footnote{A log-space formulation based on cumulative sums of
$\log(\beta_t)$ was also considered but rejected because repeated termination
resets underflow in \texttt{float32}. At the termination rate used in these
benchmarks, two to three terminations are sufficient to push the running sum
below the representable range, producing \texttt{inf}/\texttt{nan} outputs.
We therefore use the linear-space scan.} verified the same way.  All seven baselines
share the same $\log_2(T)$-doubling engine -- \texttt{parallel\_suffix\_scan}
for the backward algorithms, \texttt{parallel\_prefix\_scan} (its direct
forward-recurrence mirror) for eligibility traces and prefix sum -- so the
comparison uses one consistent baseline strength across algorithms.


\paragraph{Baseline scope} The results below are relative to this implementation; 
they do not establish optimality among PyTorch implementations. The doubling-scan 
baseline pays 6--12 kernel launches per call (one per doubling step, growing with 
\texttt{seq\_len}), against 1--2 launches for the \rltrp{} Triton kernel it is 
compared against.  This launch-count gap is architectural: 
\texttt{torch.compile}/Inductor does not fuse this elementwise-then-scan pattern 
into as few kernel launches as a hand-written Triton kernel can.

Every ``$N\times$ vs.\ \texttt{torch.compile}'' ratio reported in
Section~\ref{sec:results} is measured against this specific baseline.

\subsection{Performance Regression Guard}

The library includes a CI performance test that checks a GPU-specific minimum
speedup over the vectorized baseline at 128 environments and 1024 steps. For
each GPU, the threshold is set to 90\% of the minimum observed across three
independent runs. Across the currently calibrated H100, H200, and
RTX~2000~Ada GPUs, the lowest threshold is 1.61$\times$ (Retrace on
RTX~2000~Ada).

\section{Results}
\label{sec:results}

All results use the v0.1.3 benchmark data,\footnote{Full per-shape tables,
measurement scripts, and implementation notes are available with the library
at \url{https://github.com/simonsays1980/rl-triton}.} measured under the
protocol in Section~\ref{sec:measurement} against the baseline in
Section~\ref{sec:baseline}. We report results on an NVIDIA H100 80GB HBM3 and
an NVIDIA RTX~2000~Ada Generation. Unless stated otherwise, speedups are
full-call ratios for the complete \texttt{compute\_*(tensors) -> tensors}
invocation, including launch and wrapper overhead. Device-only ratios are
higher (1.76--13.19$\times$ at the headline configuration versus
1.6--5.70$\times$ full-call), so we report full-call ratios throughout.

\subsection{Massively-Parallel-Simulation Headline}
\label{sec:results-headline}

Table~\ref{tab:headline} reports all seven algorithms at
$\texttt{num\_envs}{=}4096$, $\texttt{seq\_len}{=}128$: the
massively-parallel-simulation regime this library targets (e.g.,
Isaac Gym/Isaac Lab-style environment counts at short rollout horizons).
It reports three tiers per algorithm -- the naive loop, the
vectorized-and-compiled scan, and the fused Triton kernel -- plus the
with-truncations path (eligibility traces and prefix sum have no
truncation-aware variant).  A sequential NumPy CPU loop is also measured
but not tabulated; it is several times slower than the GPU
loop.\footnote{E.g.\ for GAE at $\texttt{num\_envs}{=}4096$: 180ms (H100) /
56ms (RTX~2000~Ada) for the NumPy loop vs.\ 21ms / 9ms for the GPU loop.}

\begin{table}[t]
\centering
\caption{Production-regime speedup ($\texttt{num\_envs}{=}4096$,
  $\texttt{seq\_len}{=}128$), full-call wall time, v0.1.3 release
  measurements.  \emph{Loop} is the Triton kernel's speedup over the naive
  uncompiled sequential-loop baseline.  \emph{Plain} and \emph{With trunc.}
  are its speedup over the vectorized \texttt{torch.compile}
  baseline (Section~\ref{sec:baseline}).  Eligibility traces and prefix sum
  have no with-truncations path.}
\label{tab:headline}
\renewcommand{\arraystretch}{1.2}
\small
\begin{tabularx}{\textwidth}{@{}l X X X X X X@{}}
\toprule
& \multicolumn{3}{c}{\textbf{H100 80GB HBM3}} & \multicolumn{3}{c}{\textbf{RTX 2000 Ada}} \\
\textbf{Algorithm} & Loop & Plain & With trunc. & Loop & Plain & With trunc. \\
\midrule
GAE                & 317$\times$ & 2.36$\times$ & 1.7$\times$ & 281$\times$ & 2.57$\times$ & 1.6$\times$ \\
V-Trace             & 96$\times$ & 2.78$\times$ & 2.7$\times$ & 86$\times$ & 3.46$\times$ & 3.3$\times$ \\
Retrace($\lambda$)  & 72$\times$ & 1.91$\times$ & 1.9$\times$ & 16$\times$ & 1.62$\times$ & 1.6$\times$ \\
$\lambda$-returns   & 267$\times$ & 2.85$\times$ & 2.6$\times$ & 251$\times$ & 5.14$\times$ & 4.4$\times$ \\
Discounted returns  & 167$\times$ & 2.70$\times$ & 2.6$\times$ & 171$\times$ & 5.70$\times$ & 4.6$\times$ \\
Eligibility traces  & 191$\times$ & 2.48$\times$ & -- & 207$\times$ & 2.28$\times$ & -- \\
Prefix sum          & 167$\times$ & 2.39$\times$ & -- & 168$\times$ & 2.25$\times$ & -- \\
\bottomrule
\end{tabularx}
\end{table}

The 16--317$\times$ speedup over the naive loop primarily reflects the
reduction from an $O(T)$ sequential recurrence to an $O(\log T)$ associative
scan (Section~\ref{sec:scans}). The vectorized \texttt{torch.compile}
baseline already captures much of this gain. Relative to that baseline,
fusion provides a further 1.6--5.70$\times$ speedup, ranging from
1.6$\times$ for Retrace to 5.70$\times$ for discounted returns
(Table~\ref{tab:headline}).  Discounted
returns and $\lambda$-returns sit at the high end because they are
cheapest per step -- no importance-sampling ratio, fewer value-function
reads than GAE, V-Trace, or Retrace -- so the kernel does less per-step
work while the shared baseline pays the same cost regardless of algorithm.
Figure~\ref{fig:launch-overhead} shows why this regime favors fusion.

\begin{figure}[t]
\centering
\includegraphics[width=\textwidth]{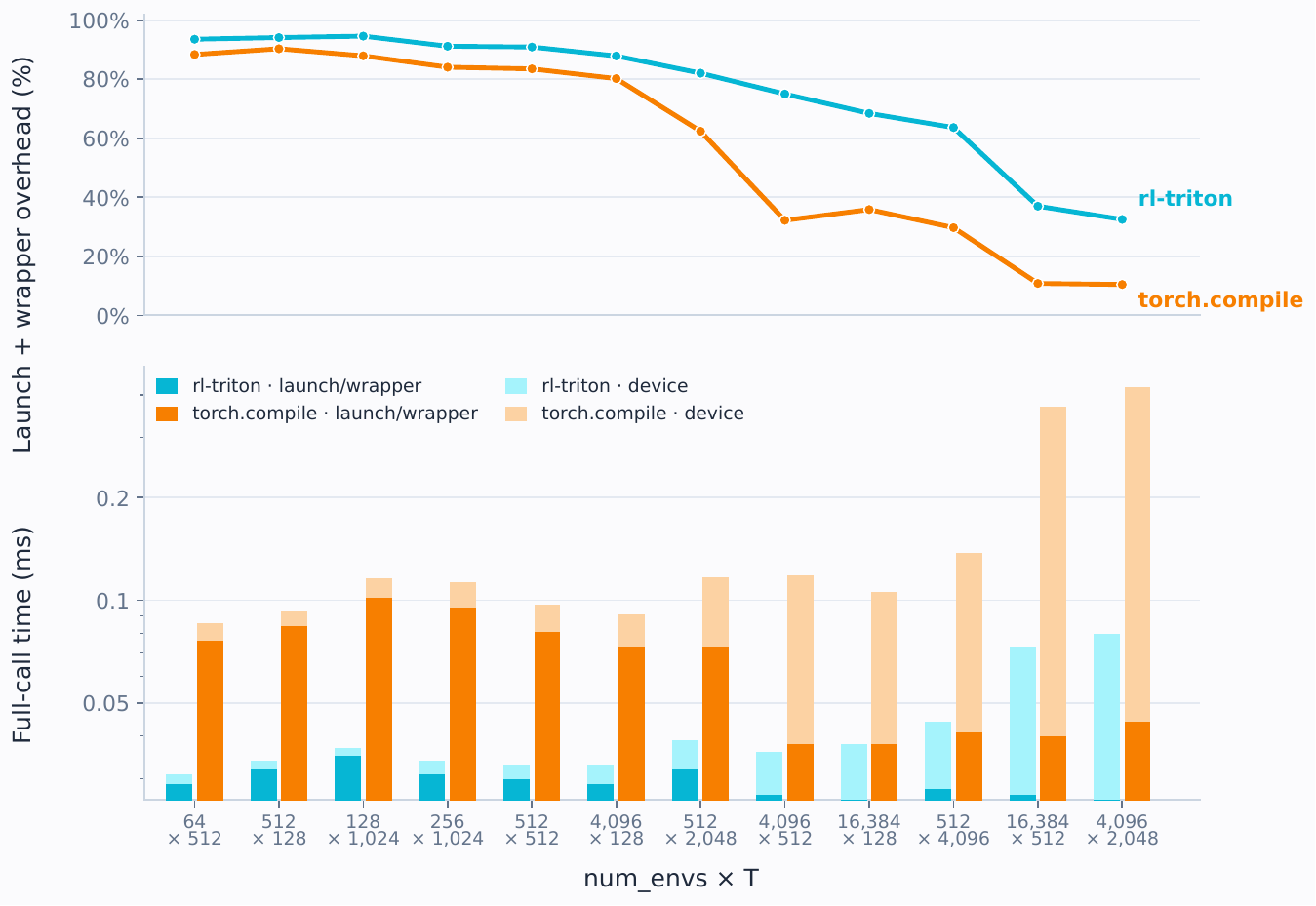}
\caption{Launch/wrapper overhead and device time for GAE (plain path, H100)
  across the twelve shapes in Table~\ref{tab:scaling}.
  \emph{Top}: launch/wrapper overhead as a fraction of full-call time for
  \rltrp{} and the \texttt{torch.compile} baseline, decreasing from
  $\sim$88--95\% at the smallest shapes to $\sim$11--37\% at the largest.
  \emph{Bottom}: full-call time decomposed into launch/wrapper and device
  components. \rltrp{}'s launch/wrapper cost remains
  $\sim$0.026--0.035\,ms, while the baseline's is 1.4--3.1$\times$ larger.}
\label{fig:launch-overhead}
\end{figure}

The smaller speedup on the truncation-aware path arises from a
property of the Triton kernel, not the baseline.  The
\texttt{torch.compile} baseline compiles and runs the with-truncations
graph unconditionally, so its own time is essentially unchanged whether
truncations are all-zero or real.  The Triton kernel, by contrast, uses 
a \texttt{HAS\_TRUNCATIONS} compile-time fast path
(Section~\ref{sec:fused}) that skips the truncateds array and the 2D
bootstrap-values tensor when there are no truncations.  For GAE, V-Trace, 
$\lambda$-returns, and discounted returns, the plain path is
consistently faster in absolute terms and therefore generally achieves a
larger speedup over \texttt{torch.compile}. Retrace shows little systematic
difference between the two paths.

Retrace($\lambda$) gives the smallest speedup at this shape on both GPUs,
reflecting its heavier per-timestep work relative to the other algorithms.
Its with-truncations ratio is still smaller than its plain ratio, like every
other algorithm (1.91$\times \to$ 1.9$\times$ on H100,
1.62$\times \to$ 1.6$\times$ on RTX~2000~Ada), but by a much smaller gap:
the savings from the \texttt{HAS\_TRUNCATIONS} fast path constitute a smaller
fraction of its total runtime. The same effect also narrows its speedup over
the sequential loop, whose HBM round-trips account for a smaller fraction of
the heavier computation.

The relative H100 and RTX speedups are shape-dependent: they vary by algorithm
and can reverse between shapes for the same algorithm
(Section~\ref{sec:results-limitations}). Neither GPU consistently yields a
larger Triton-to-\texttt{torch.compile} speedup.

This massively-parallel-simulation shape ($\texttt{num\_envs}{=}4096$,
$\texttt{seq\_len}{=}128$) is short by design, and it is not where the
kernel's advantage is largest.
The $O(\log T)$-depth argument
(Section~\ref{sec:scans}) predicts the opposite: longer rollouts should
benefit \emph{more}, since the vectorized baseline pays proportionally more
$O(\log T)$ HBM round-trips as $T$ grows while the fused kernel avoids 
intermediate HBM round-trips between scan stages.  We therefore evaluate 
sequence lengths up to 4096; Section~\ref{sec:results-scaling} reports the 
resulting scaling behavior.

\subsection{Full-Grid Scaling: GAE and Retrace}
\label{sec:results-scaling}

Table~\ref{tab:scaling} reports the full shape sweep (twelve
$(\texttt{num\_envs}, \texttt{seq\_len})$ configurations, spanning
$\texttt{num\_envs} \in \{64,\ldots,16384\}$ and
$\texttt{seq\_len} \in \{128,\ldots,4096\}$) for GAE and Retrace($\lambda$) on
both GPUs, plain path, full-call ratio. The corresponding sweeps for the 
remaining five algorithms are reported in Appendix~\ref{sec:appendix-scaling}.

\begin{table}[t]
\centering
\caption{Full shape sweep, plain path, full-call speedup vs.\ the
  vectorized \texttt{torch.compile} baseline (v0.1.3 release
  measurements).}
\label{tab:scaling}
\renewcommand{\arraystretch}{1.15}
\small
\begin{tabularx}{\textwidth}{@{}X X X X X X@{}}
\toprule
& & \multicolumn{2}{c}{\textbf{H100 80GB HBM3}} & \multicolumn{2}{c}{\textbf{RTX 2000 Ada}} \\
\textbf{num\_envs} & \textbf{seq\_len} & GAE & Retrace & GAE & Retrace \\
\midrule
64    & 512  & 2.7$\times$ & 2.2$\times$ & 2.7$\times$  & 2.0$\times$ \\
128   & 1024 & 3.2$\times$ & 2.3$\times$ & 3.1$\times$  & 1.8$\times$ \\
256   & 1024 & 3.3$\times$ & 2.3$\times$ & 3.0$\times$  & 1.4$\times$ \\
512   & 128  & 2.8$\times$ & 2.0$\times$ & 2.5$\times$  & 1.8$\times$ \\
512   & 512  & 2.9$\times$ & 2.2$\times$ & 2.8$\times$  & 1.4$\times$ \\
512   & 2048 & 3.0$\times$ & 1.4$\times$ & 5.7$\times$  & 2.0$\times$ \\
512   & 4096 & 3.2$\times$ & 0.6$\times$ & 6.3$\times$  & 0.8$\times$ \\
4096  & 128  & 2.8$\times$ & 1.9$\times$ & 2.5$\times$  & 1.6$\times$ \\
4096  & 512  & 3.3$\times$ & 1.6$\times$ & 5.3$\times$  & 2.4$\times$ \\
4096  & 2048 & 5.3$\times$ & 1.8$\times$ & 7.8$\times$  & 2.9$\times$ \\
16384 & 128  & 2.8$\times$ & 1.9$\times$ & 4.8$\times$  & 2.3$\times$ \\
16384 & 512  & 5.1$\times$ & 1.9$\times$ & 6.8$\times$  & 2.7$\times$ \\
\bottomrule
\end{tabularx}
\end{table}

GAE remains above $2.5\times$ across the sweep and reaches its largest ratios
at the larger shapes.  Against the uncompiled sequential-loop baseline instead (not tabulated here
at every shape; see Table~\ref{tab:headline} for one representative shape),
GAE's gain grows from 1464$\times$ at $(64,512)$ to 7570$\times$ at
$(512,4096)$ on H100. The loop baseline executes an $O(T)$ serial chain of
GPU operations, with repeated kernel launches and global-memory traffic at
each timestep, whereas the associative scan has only $O(\log T)$ dependency
depth.  Retrace 
behaves differently: at $\texttt{seq\_len}{=}4096$, its speedup over
the doubling-scan baseline falls below $1\times$ (0.6$\times$ on H100 and
0.8$\times$ on RTX~2000~Ada). It remains faster than the uncompiled
sequential loop at every shape in the grid; its minimum loop-baseline speedup
is 43.5$\times$ on H100 and 4.6$\times$ on RTX~2000~Ada.
Section~\ref{sec:results-retrace} analyzes the source of the regression.

\subsection{Retrace's Register-Pressure Regression Above \texorpdfstring{$\texttt{seq\_len}{=}2048$}{seq\_len=2048}}
\label{sec:results-retrace}

Retrace($\lambda$)'s fused kernel reads the
$[\texttt{num\_envs},\texttt{seq\_len},\texttt{num\_actions}]$
action-probability tensor a second time in-kernel to construct $c_{t+1}$,
leaving both the original and shifted values live during the scan. At
$\texttt{seq\_len}{=}4096$, this increases register demand enough to cause
spilling: at $\texttt{seq\_len}{=}4096$ ($\texttt{num\_envs}{=}512$,
$\texttt{num\_actions}{=}4$), \texttt{ptxas} reports 128 registers per thread,
a 400-byte per-thread stack frame, and 456 bytes each of spill stores and
spill loads, at 25\% occupancy.\footnote{\texttt{ptxas} is NVIDIA's
PTX-to-SASS assembler and performs register allocation during compilation;
registers that cannot be allocated are spilled to local memory, with spill
traffic reported in bytes. On H100, each SM provides 64K 32-bit registers
and supports at most 64 resident warps. At 128 registers per thread, one
32-thread warp consumes 4096 registers, limiting an SM to 16 resident warps,
or 25\% occupancy from register pressure alone. The spill traffic does not
enter this occupancy calculation. Measurements use Triton 3.0.0 and
\texttt{ptxas} 12.4.99 targeting \texttt{sm\_90a}; register allocation is
compiler-version dependent
(\texttt{benchmarks/measure\_retrace\_register\_pressure.py}).}

The implementation therefore dispatches longer sequences to the generic
materialize-then-scan path, which computes the 3D probability terms outside
the scan kernel and avoids the same register-pressure limit. The
$\texttt{seq\_len}{=}2048$ threshold is conservative rather than the
crossover between the two \rltrp{} paths: at $\texttt{seq\_len}{=}4096$ the
fused path is still faster than the generic path (0.73$\times$ vs.\
0.47$\times$ relative to \texttt{torch.compile}), whereas the generic path
overtakes it at $\texttt{seq\_len}{=}8192$.

\subsection{End-to-End PPO Speedup}
\label{sec:results-ppo}

We evaluate whether the kernel-level speedups translate into end-to-end PPO
speedups using a synthetic update ($\texttt{seq\_len}{=}128$, 4 epochs
$\times$ 4 minibatches) on an H100. We compare against both the sequential
backward loop used by CleanRL, RLlib, and Sample Factory
(Section~\ref{sec:relatedwork}) and the vectorized \texttt{torch.compile}
baseline of Section~\ref{sec:baseline}, at two policy sizes:

\begin{table}[t]
\centering
\caption{GAE share of a PPO update and end-to-end Triton speedup over each baseline, 
  by policy hidden size ($\texttt{seq\_len}{=}128$).
}
\label{tab:ppo-e2e}
\renewcommand{\arraystretch}{1.2}
\small
\begin{tabularx}{\textwidth}{@{}l l X X X@{}}
\toprule
\textbf{Hidden size} & \textbf{num\_envs} & GAE share / E2E vs.\ loop & GAE share / E2E vs.\ scan \\
\midrule
(1024, 1024) & 4096  & ${\sim}2.2\%$ / ${\sim}1.02\times$ & $<0.1\%$ / ${\sim}1.00\times$ (within noise) \\
(128, 128)   & 16384 & ${\sim}10$--$15\%$ / ${\sim}1.11$--$1.16\times$ & not measured at this configuration$^\dagger$ \\
\bottomrule
\end{tabularx}

\vspace{2pt}
{\footnotesize $^\dagger$The (128, 128) configuration was evaluated only against the loop baseline.}
\end{table}

At the (1024, 1024) policy size, GAE accounts for ${\sim}2.2\%$ of the PPO
update against the loop baseline and less than $0.1\%$ against the vectorized
scan. The corresponding end-to-end speedups are ${\sim}1.02\times$ and
${\sim}1.00\times$, respectively, despite the larger isolated-kernel gains in
Table~\ref{tab:headline}. At (128, 128) (${\sim}150$k parameters), GAE
accounts for 10--15\% of the update against the loop baseline, and the
end-to-end speedup rises to 1.11--1.16$\times$ across nine reruns.

These results closely follow Amdahl's law: eliminating a component that
accounts for 2.2\% of an update can improve total runtime by at most
${\sim}1.02\times$, whereas a 10--15\% share permits approximately
1.11--1.18$\times$. The microbenchmark results in
Sections~\ref{sec:results-headline}--\ref{sec:results-scaling} should
therefore be interpreted as isolated-kernel speedups rather than equivalent
gains in training throughput.

For feed-forward policies, the addressable share can increase with longer
rollouts because policy computation batches timesteps independently whereas
credit assignment retains a temporal dependency; repeated advantage
computation increases it further. Credit assignment is also learner-side work
in high-throughput actor--learner methods such as IMPALA
\citep{espeholt2018impala} and APPO \citep{petrenko2020samplefactory}, while
PufferLib \citep{suarez2025pufferlib} similarly targets a small-policy,
high-throughput regime with a custom CUDA advantage kernel
\citep{pufferlib2025blog}. At (128, 128), CUDA graph replay improves the
optimizer step by only ${\sim}1.02$--$1.08\times$, indicating that launch
overhead is not the dominant cost.

\subsection{Comparison Against PufferLib}
\label{sec:results-pufferlib}

We compare \rltrp{} against PufferLib's hand-written CUDA advantage kernel
(\texttt{puff\_advantage\_row\_cuda}), which assigns one thread per
environment row and performs a sequential $O(T)$ scan. Measurements were
collected on an H100 after verifying both kernels against an independent
reference on all jointly supported outputs.\footnote{The benchmark uses a 
pinned, sha256-verified copy of PufferLib
3.0.0's CUDA extension source, JIT-compiled
(\texttt{benchmarks/pufferlib\_ext/}). Timing uses CUDA events with warmup,
explicit synchronization, and the minimum of 11 per-trial medians
(100 iterations/trial) per configuration; both kernels are checked against
an independent reference before any timing is trusted. Full methodology,
the equivalence proof, and the corresponding V-Trace comparison are
provided with the \rltrp{} benchmark suite.}

We evaluate two workload regimes. The \emph{classic on-policy} regime
($\texttt{seq\_len}\in\{128,\ldots,4096\}$,
$\texttt{num\_envs}\in\{128,\ldots,8192\}$, Figure~\ref{fig:pufferlib-classic}) represents moderate-to-long
rollouts with moderate parallelism, as commonly used in PPO-style training.
The \emph{massively parallel simulation} regime
($\texttt{seq\_len}\in\{8,\ldots,128\}$,
$\texttt{num\_envs}\in\{4096,\ldots,32768\}$, Figure~\ref{fig:pufferlib-mps}) covers short rollouts with
high environment parallelism. This regime appears in systems such as
PufferLib \citep{suarez2025pufferlib}, Gigaflow
\citep{cusumanotowner2025gigaflow}, GPUDrive
\citep{kazemkhani2024gpudrive}, Nocturne \citep{vinitsky2022nocturne},
Isaac Gym/Isaac Lab \citep{makoviychuk2021isaacgym,isaaclab2025}, and
robot-locomotion PPO workloads surveyed by UniLab \citep{jia2026unilab}.

\begin{figure}[t]
\centering
\includegraphics[width=\textwidth]{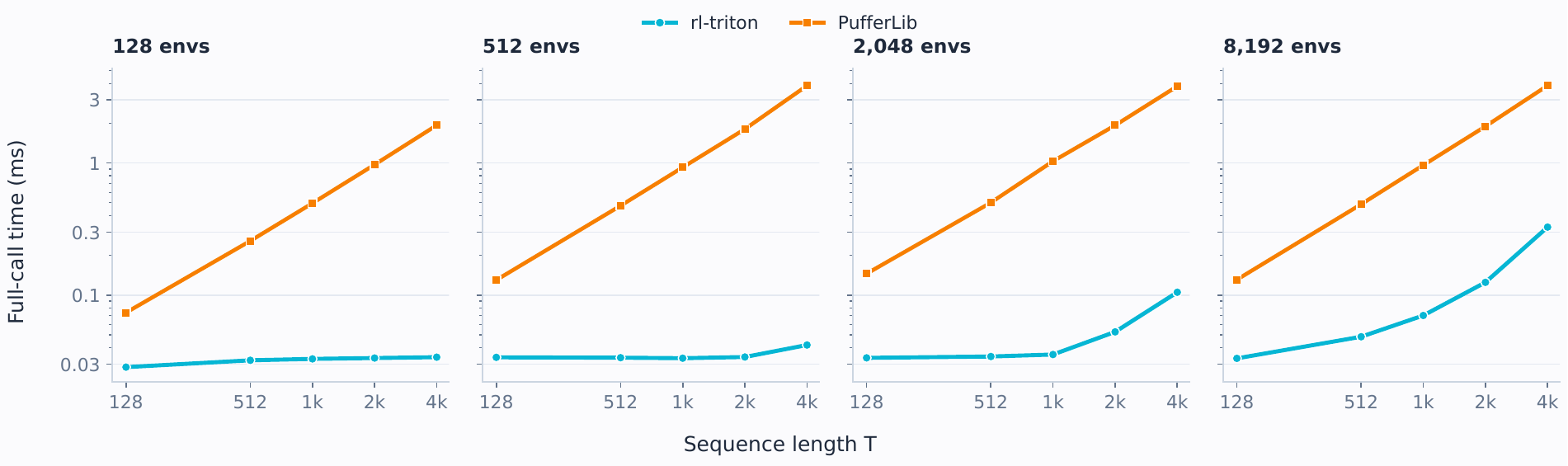}
\caption{Full-call GAE performance of \rltrp{} and PufferLib in the classic 
  on-policy regime on an H100. \rltrp{} has lower runtime across all 20 
  evaluated shapes.
}
\label{fig:pufferlib-classic}
\end{figure}

\begin{figure}[t]
\centering
\includegraphics[width=\textwidth]{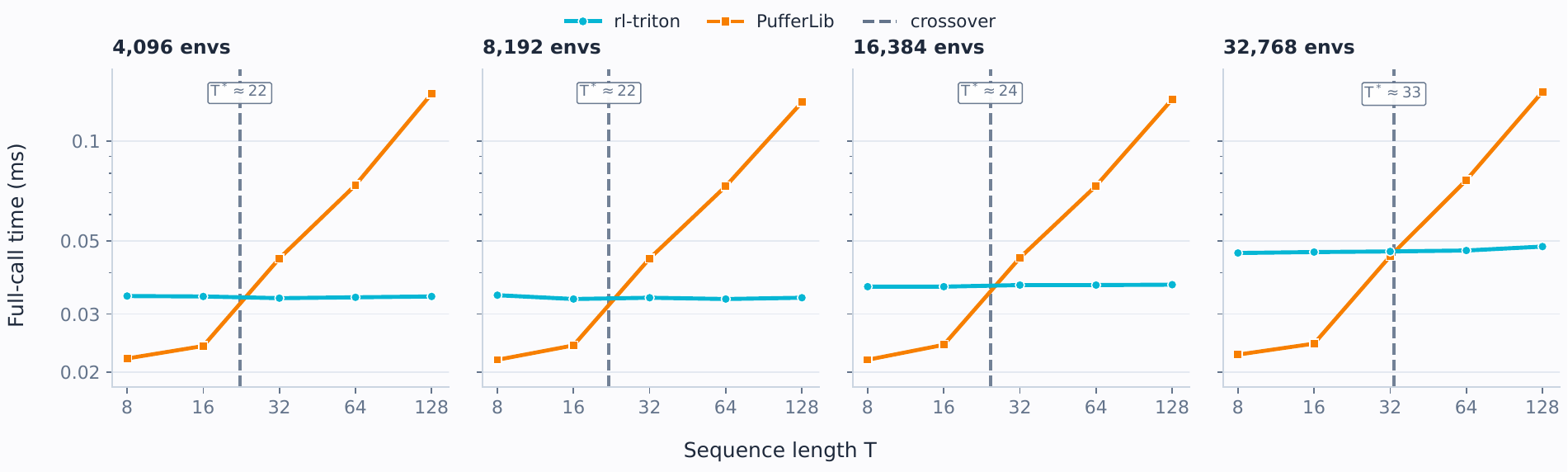}
\caption{Full-call GAE performance of \rltrp{} and PufferLib in the massively parallel
simulation regime on an H100. \rltrp{} has lower runtime at 11 of 20 shapes,
including every configuration with $\texttt{seq\_len}\ge 32$ except
$(\texttt{num\_envs},\texttt{seq\_len})=(32768,32)$; PufferLib is faster at
all configurations with $\texttt{seq\_len}\le 16$. The crossover shifts to
longer sequences as $\texttt{num\_envs}$ increases, reflecting PufferLib's
low overhead for very short per-thread sequential scans.}
\label{fig:pufferlib-mps}
\end{figure}

Across both regimes, \rltrp{} has lower full-call time at 31 of 40 evaluated
shapes. Sequence length largely determines the crossover: in the massively
parallel simulation regime, changing $\texttt{num\_envs}$ does not change the
winner at any tested $\texttt{seq\_len}$ except
$(\texttt{num\_envs},\texttt{seq\_len})=(32768,32)$
(Figure~\ref{fig:pufferlib-mps}). PufferLib exposes a single \texttt{dones}
flag per step and provides neither termination/truncation semantics nor a
bootstrap-value mechanism, so these cases cannot be compared directly.

\section{Related Work}
\label{sec:relatedwork}

\paragraph{RL frameworks.}
CleanRL \citep{huang2022cleanrl}, RLlib \citep{liang2018rllib}, and Sample
Factory \citep{petrenko2020samplefactory} implement GAE and return computation
in their training stacks. Representative PyTorch implementations evaluate the
temporal recurrence with sequential backward loops rather than an associative
scan. \rltrp{} provides GPU implementations with distinct handling of
termination, truncation, and rollout-window bootstrapping
(Section~\ref{sec:boundaries}).

\paragraph{Parallel scans in ML.}
Associative scans are widely used to parallelize linear recurrences in sequence
models. State-space and linear recurrent architectures such as S4
\citep{gu2022s4}, S5 \citep{smith2023s5}, LRU \citep{orvieto2023lru}, and Mamba
\citep{gu2023mamba} exploit related parallel recurrence formulations.
LRUs have also been used as recurrent backbones for partially observable RL.
\rltrp{} instead applies associative scans to post-rollout scalar credit
assignment rather than hidden-state sequence modeling.

\paragraph{IO-aware GPU kernels.}
FlashAttention \citep{dao2022flashattention,dao2023flashattention2}
demonstrates how tiling and fusion around the GPU memory hierarchy can reduce
HBM traffic and accelerate memory-intensive ML operators. \rltrp{} applies the
same IO-aware principle to RL credit assignment, using Triton to keep the
associative scan on-chip.

\paragraph{Off-policy correction.}
V-Trace \citep{espeholt2018impala} and Retrace($\lambda$)
\citep{munos2016retrace} are the off-policy correction algorithms implemented
in \rltrp{}. ACER \citep{wang2017acer} uses related truncated
importance-sampling corrections with trust-region stabilization but is not
currently implemented. These methods contain recursive importance-weighted
return computations to which the affine-scan formulation applies.

\section{Discussion and Future Work}

\subsection{Applicability Beyond Standard RL}

The recurrence implemented by \rltrp{} also appears outside conventional
on-policy RL, although direct applications beyond RL are limited.

\begin{itemize}

\item \textbf{Long-context LLM post-training.}
PPO-style LLM training applies credit assignment over generated tokens, so
long-chain-of-thought responses can turn advantage estimation into a
thousands-step recurrence. T-PPO computes GAE-style advantages over
8K-token truncated windows within responses of up to 24K tokens
\citep{fan2025truncated}, while recent long-context PPO experiments evaluate
token-level GAE with up to 8192-token responses and 4096 rollouts per update
\citep{gong2026segmental}. The same pattern appears in current LLM-RL implementations: SkyRL
\citep{cao2025skyrl} and \texttt{verl} \citep{sheng2024hybridflow} implement
token-level GAE with reverse sequential loops, while their REINFORCE++
estimators likewise compute discounted returns by scanning backward over
response tokens. These implementations expose the same GAE and discounted-return 
recurrences targeted by the fused scans here, with framework-specific masking semantics.

\item \textbf{VLA reinforcement learning.}
PPO-based post-training of vision-language-action policies requires advantage
estimation over environment interaction trajectories. RLinf-VLA
\citep{zang2025rlinf}, for example, supports PPO training of VLA policies
across robotic simulators, and its current implementation computes GAE over
$[T,B]$ rollout tensors with a reverse Python loop over $T$. This is exactly 
the GAE recurrence implemented by the fused scan here,
making VLA training with many parallel environments and longer interaction
horizons another natural application.

\item \textbf{World-model reinforcement learning.}
DreamerV4 \citep{hafner2025dreamer4} trains its value function on imagined
trajectories using bootstrapped $\lambda$-returns and derives its policy
advantages from the same targets. These returns obey the backward
affine recurrence implemented here, making imagination-based world-model
training another direct application of the fused $\lambda$-return scan.

\item \textbf{Spiking-neural-network BPTT.}
A direct non-RL instance occurs in backpropagation through hard-reset spiking
neurons when reset differentiation is detached. The forward pass generates a
runtime spike sequence $S_t$ from the neuron state; spike times, and hence the
resulting segment lengths, are not known in advance. The backward temporal
gradient then has the affine form
\[
    g_t = \alpha_t + \beta_t g_{t+1},
    \qquad
    \beta_t = \kappa_t(1-S_t),
\]
so each spike zeros the temporal carry. SNN frameworks such as SpikingJelly
provide dedicated GPU kernels for multi-step neuron forward and backward
passes \citep{fang2023spikingjelly}, making this a GPU-relevant instance of
the same segmented affine recurrence.

\end{itemize}

Direct applicability requires a scalar affine recurrence and enough independent
sequences for GPU parallelism. RL credit assignment additionally attaches reset
and bootstrap semantics to per-step boundary signals. Hard-reset
spiking-neural-network BPTT with detached reset differentiation is a non-RL
example with comparable recurrence-level boundary semantics.

\subsection{Limitations}
\label{sec:results-limitations}

\paragraph{Retrace register pressure above
\texorpdfstring{$\texttt{seq\_len}{=}2048$}{seq\_len=2048}.}
Retrace($\lambda$) becomes slower than \texttt{torch.compile} above
$\texttt{seq\_len}{=}2048$ because of the register-pressure effects described
in Section~\ref{sec:results-retrace}.

\paragraph{GAE device-time regression at the shortest,
highest-parallelism shape.}
At $\texttt{num\_envs}{=}16384$ and $\texttt{seq\_len}{=}16$ on H100,
GAE's device-only ratio falls below $1\times$ relative to
\texttt{torch.compile}, while its full-call ratio remains above $1\times$
because launch and wrapper overhead dominate at this shape. GAE is the only
one of the five measured algorithms that exhibits this inversion.

\paragraph{Warp-count tuning gap above \texttt{BLOCK\_SIZE=16384}.}
Per-\texttt{BLOCK\_SIZE} warp counts are tuned only up to
$\texttt{BLOCK\_SIZE}{=}16384$; larger, reachable sizes fall back to an
untuned default of 16 warps.  Tuning at small \texttt{BLOCK\_SIZE} recovered
2.1--2.7$\times$ over an untuned default there, suggesting that larger block 
sizes may also benefit from tuning. We leave
this evaluation to future work.

\paragraph{The H100-vs-RTX margin direction is shape-dependent -- open question.}
At a small, launch-overhead-dominated shape (128 environments, 1024 steps)
vs.\ the massively-parallel-simulation shape of Table~\ref{tab:headline},
GAE, V-Trace, and $\lambda$-returns each reverse which card shows the
larger Triton-vs-\texttt{torch.compile} margin; Retrace alone favors H100
at both.  This reversal is a shape effect rather than evidence of a
general "consumer vs.\ datacenter" trend; its underlying mechanism remains
unidentified.

\paragraph{Baseline dependence.}
All reported speedups are relative to the baseline described in
Section~\ref{sec:baseline}. A faster numerically stable PyTorch
implementation would reduce these ratios.

\paragraph{Sequence length limit.}
Every flat fused kernel is limited to $\texttt{seq\_len} \le 131072$ by
Triton's internal block size constraints; the backward kernels fall back to
an unfused, two-pass chunked path beyond that, but the two forward scans
(eligibility traces, episodic prefix sum) have no chunked fallback at all.

\paragraph{Float32 only.}
All kernels require \texttt{float32} inputs (Section~\ref{sec:precision});
\texttt{bfloat16} support is deferred to a future release pending a
precision-tradeoff analysis across algorithms and sequence lengths.

\paragraph{Scalar $\beta$ only.}
$\beta_t$ must be a scalar per timestep, shared across whatever is scanned.
Per-dimension $\beta_t$, as true online TD($\lambda$)/per-parameter
eligibility traces require \citep{vanseijen2016trueonline, sutton2018rl},
is not supported: $\beta_t$ is baked into the kernel launch as a scalar
rather than read per-lane from memory, and supporting it would require a 
different kernel structure.

\subsection{Planned Extensions}

Planned work focuses on broader rollout support and lower-overhead execution.
For multi-turn and tool-use trajectories, we plan observation-skip masking
that carries the recurrence across masked tokens ($\beta_t=1$), together with
fused advantage normalization and value loss. Systems work includes
\texttt{bfloat16} I/O, chunked forward scans, CUDA graph support, and fixing
the Retrace register-pressure regression
(Section~\ref{sec:results-retrace}). We also plan differentiable PPO, GRPO,
and KL-loss kernels with hand-written backward passes.

\section{Conclusion}

The seven credit-assignment algorithms considered here -- GAE, V-Trace,
Retrace($\lambda$), TD($\lambda$) returns, discounted returns, eligibility
traces, and episodic prefix sums -- share a common first-order linear
recurrence structure that can be evaluated with the same associative scan
operator. Each algorithm is implemented by its own fused Triton kernel,
which constructs its recurrence coefficients on-chip. By keeping the
$O(\log T)$ scan on-chip, \rltrp{} avoids the sequential loop's per-timestep
HBM round-trips and the vectorized \texttt{torch.compile} baseline's
intermediate HBM round-trips between doubling stages. In the massively
parallel simulation regime, this yields 1.6--5.70$\times$ full-call speedups
over the vectorized baseline. For most algorithms, the gains increase at
longer sequence lengths; Retrace is the exception because of its
long-sequence register-pressure regression. The implementation preserves the
specified handling of terminated episodes, truncated episodes, and
rollout-window boundaries.


\section*{Acknowledgements}
The author thanks Artur Niederfahrenhorst (Anyscale) for valuable feedback
and discussions, and Mark Towers (Gymnasium maintainer) for comments on 
an early draft and support with the arXiv submission.

\bibliographystyle{plainnat}
\bibliography{paper}

\appendix
\section{Four-Thread Scan Trace}
\label{sec:appendix-scan-trace}

This appendix gives the complete thread-by-thread combine underlying the
sketch in Section~\ref{sec:algebraic-verification}.  The array is loaded in
reverse chronological order, so thread~$i$ ($i=1,\ldots,4$) initially holds
the tuple $(\alpha_{4-i}, \beta_{4-i})$.  We write $T_{a..b}$ for
the accumulated tuple covering reversed positions $a$ through $b$, and
$T_{k}$ as shorthand for the single-position tuple $T_{k..k}$ -- thread~$k$'s
value \emph{before any reduction pass runs} (position 1 = last timestep
$t{=}3$, position 4 = first timestep $t{=}0$):
\[
  T_{1} = (\alpha_3,\beta_3), \quad
  T_{2} = (\alpha_2,\beta_2), \quad
  T_{3} = (\alpha_1,\beta_1), \quad
  T_{4} = (\alpha_0,\beta_0).
\]

\paragraph{Reduction pass 1} (stride 1). Each thread reads from the thread
exactly one position to its left.  Thread~1 has no left neighbour and keeps
its initial tuple.  Threads 2, 3, and 4 each combine with their immediate
left neighbour's \emph{initial} value $T_k$ (all reads happen simultaneously
before any writes, so every combine below uses the single-position tuples
just defined, never an already-updated span):
\begin{align*}
  T_{1..1} &= T_1 = (\alpha_3,\;\beta_3), \\
  T_{1..2} &= T_2 \oplus T_1
             = (\alpha_2 + \beta_2\alpha_3,\;\beta_3\beta_2), \\
  T_{2..3} &= T_3 \oplus T_2
             = (\alpha_1 + \beta_1\alpha_2,\;\beta_2\beta_1), \\
  T_{3..4} &= T_4 \oplus T_3
             = (\alpha_0 + \beta_0\alpha_1,\;\beta_1\beta_0).
\end{align*}
After pass 1, threads 2, 3, and 4 each cover two adjacent positions.

\paragraph{Reduction pass 2} (stride 2). Each thread reads from the thread
two positions to its left.  Thread~4 combines its current tuple $T_{3..4}$
with thread~2's result $T_{1..2}$:
\[
  T_{1..4} = T_{3..4} \oplus T_{1..2}
  = \bigl(\alpha_0 + \beta_0\alpha_1 + \beta_0\beta_1\alpha_2
          + \beta_0\beta_1\beta_2\alpha_3,
    \;\beta_3\beta_2\beta_1\beta_0\bigr).
\]
Thread~4 now covers all four positions, and the loop terminates -- exactly
$\log_2(\texttt{BLOCK\_SIZE}) = \log_2 4 = 2$ fixed passes, no dynamic
stopping criterion.  Each thread now holds the correct tuple for its own
position: thread~3's, for instance, is
$T_{2..3} \oplus T_{1..1} = T_{1..3}$ by the same stride-2 combine.

\section{Full-Grid Scaling: Remaining Algorithms}
\label{sec:appendix-scaling}

Table~\ref{tab:scaling} (Section~\ref{sec:results-scaling}) reports the
twelve-shape sweep for GAE and Retrace($\lambda$), including Retrace's
long-sequence regression analyzed in Section~\ref{sec:results-retrace}.
Table~\ref{tab:scaling-appendix} reports the corresponding results for the
remaining five algorithms.

\begin{table}[t]
\centering
\caption{Full twelve-shape sweep for V-Trace, TD($\lambda$)/
  $\lambda$-returns, discounted returns, eligibility traces, and episodic
  prefix sum. Values are full-call speedups over the vectorized
  \texttt{torch.compile} baseline using the v0.1.3 release measurements.}
\label{tab:scaling-appendix}
\renewcommand{\arraystretch}{1.1}
\scriptsize
\begin{tabularx}{\textwidth}{@{}X X X X X X X X X X X X@{}}
\toprule
& & \multicolumn{5}{c}{\textbf{H100 80GB HBM3}} & \multicolumn{5}{c}{\textbf{RTX 2000 Ada}} \\
\textbf{num\_envs} & \textbf{seq\_len} & V-Trace & $\lambda$-ret. & Disc.\ ret. & Elig. & Prefix & V-Trace & $\lambda$-ret. & Disc.\ ret. & Elig. & Prefix \\
\midrule
64    & 512  & 3.0$\times$ & 3.0$\times$ & 3.0$\times$ & 2.8$\times$ & 2.8$\times$ & 2.9$\times$  & 2.9$\times$  & 2.9$\times$  & 2.6$\times$ & 2.5$\times$ \\
128   & 1024 & 3.0$\times$ & 3.3$\times$ & 3.3$\times$ & 2.7$\times$ & 2.7$\times$ & 2.9$\times$  & 3.1$\times$  & 3.1$\times$  & 2.6$\times$ & 2.5$\times$ \\
256   & 1024 & 3.0$\times$ & 3.2$\times$ & 3.2$\times$ & 2.8$\times$ & 2.8$\times$ & 2.9$\times$  & 3.6$\times$  & 3.3$\times$  & 2.6$\times$ & 2.6$\times$ \\
512   & 128  & 2.9$\times$ & 2.9$\times$ & 2.8$\times$ & 2.5$\times$ & 2.4$\times$ & 2.8$\times$  & 2.7$\times$  & 2.6$\times$  & 2.3$\times$ & 2.2$\times$ \\
512   & 512  & 3.1$\times$ & 3.1$\times$ & 3.1$\times$ & 2.8$\times$ & 2.7$\times$ & 3.0$\times$  & 3.3$\times$  & 3.3$\times$  & 2.6$\times$ & 2.5$\times$ \\
512   & 2048 & 3.0$\times$ & 3.2$\times$ & 3.3$\times$ & 3.1$\times$ & 3.2$\times$ & 4.2$\times$  & 16.8$\times$ & 17.2$\times$ & 4.9$\times$ & 4.6$\times$ \\
512   & 4096 & 2.9$\times$ & 5.9$\times$ & 5.8$\times$ & 3.2$\times$ & 2.9$\times$ & 6.2$\times$  & 15.0$\times$ & 24.5$\times$ & 15.6$\times$ & 15.1$\times$ \\
4096  & 128  & 2.8$\times$ & 2.9$\times$ & 2.8$\times$ & 2.5$\times$ & 2.5$\times$ & 3.5$\times$  & 5.5$\times$  & 5.7$\times$  & 2.3$\times$ & 2.3$\times$ \\
4096  & 512  & 3.1$\times$ & 4.9$\times$ & 5.3$\times$ & 2.7$\times$ & 2.7$\times$ & 5.7$\times$  & 11.8$\times$ & 27.9$\times$ & 10.4$\times$ & 8.8$\times$ \\
4096  & 2048 & 4.8$\times$ & 8.5$\times$ & 8.0$\times$ & 5.4$\times$ & 5.4$\times$ & 7.2$\times$  & 14.9$\times$ & 18.1$\times$ & 7.8$\times$ & 7.9$\times$ \\
16384 & 128  & 3.0$\times$ & 4.4$\times$ & 4.2$\times$ & 2.0$\times$ & 2.2$\times$ & 5.4$\times$  & 10.0$\times$ & 21.8$\times$ & 8.5$\times$ & 8.0$\times$ \\
16384 & 512  & 5.2$\times$ & 9.0$\times$ & 8.8$\times$ & 4.8$\times$ & 4.7$\times$ & 6.6$\times$  & 12.9$\times$ & 15.6$\times$ & 6.5$\times$ & 6.5$\times$ \\
\bottomrule
\end{tabularx}
\end{table}

Unlike Retrace, none of the remaining five algorithms falls below
$1\times$ at any measured shape. At several long-sequence configurations,
the RTX~2000~Ada ratios are substantially larger than on H100. This is
driven mainly by the \texttt{torch.compile} baseline: for discounted returns
at $(512,4096)$, the speedup is 5.8$\times$ on H100 and 24.5$\times$ on RTX,
while the baseline is 11$\times$ slower on RTX (2.62\,ms vs.\ 0.24\,ms) and
the Triton kernel itself only ${\sim}2.6\times$ slower.

\end{document}